\documentclass{article}
\pdfoutput=1

\usepackage{arxiv}

\usepackage[utf8]{inputenc} 
\usepackage[T1]{fontenc}    
\usepackage{hyperref}       
\usepackage{url}            
\usepackage{booktabs}       
\usepackage{amsfonts}       
\usepackage{nicefrac}       
\usepackage{microtype}      
\usepackage{lipsum}		
\usepackage[pdftex]{graphicx}
\usepackage[square,sort,comma,numbers]{natbib}
\usepackage{doi}

\usepackage{graphicx}

\usepackage[table]{xcolor} 
\usepackage{xcolor}
\usepackage{amsmath}
\usepackage{tcolorbox}
\usepackage{multirow}
\usepackage{layouts}
\usepackage{enumitem}

\usepackage{subfigure}
\usepackage{wrapfig}

\newcommand{\LightSecAgg}{\texttt{LightSecAgg}}
\newcommand{\LightSecAggspace}{\LightSecAgg{ }}
\newcommand{\name}{\texttt{BASecAgg}}
\newcommand{\namespace}{\name{ }}
\newcommand{\secagg}{\texttt{SecAgg}}
\newcommand{\secaggspace}{\secagg{ }}
\newcommand{\FedAsync}{\texttt{FedAsync}} \newcommand{\FedAsyncspace}{\FedAsync{ }}
\newcommand{\FedAt}{\texttt{FedAt}}
\newcommand{\FedAtspace}{\FedAt{ }}
\newcommand{\FedAvg}{\texttt{FedAvg}}
\newcommand{\FedAvgspace}{\FedAvg{ }}
\newcommand{\FedBuff}{\texttt{FedBuff}}
\newcommand{\FedBuffspace}{\FedBuff{ }}

\usepackage{algorithm}
\usepackage[noend]{algpseudocode}
\usepackage{dsfont}

\usepackage{enumitem}

\newlist{SubItemList}{itemize}{1}
\setlist[SubItemList]{label={$-$}}

\let\OldItem\item
\newcommand{\SubItemStart}[1]{%
    \let\item\SubItemEnd
    \begin{SubItemList}[resume]%
        \OldItem #1%
}
\newcommand{\SubItemMiddle}[1]{%
    \OldItem #1%
}
\newcommand{\SubItemEnd}[1]{%
    \end{SubItemList}%
    \let\item\OldItem
    \item #1%
}
\newcommand*{\SubItem}[1]{%
    \let\SubItem\SubItemMiddle%
    \SubItemStart{#1}%
}%

\usepackage{babel}
\usepackage[font=small,labelfont=bf]{caption}

\usepackage{mathtools}

\DeclarePairedDelimiter{\floor}{\lfloor}{\rfloor}

\DeclareMathAlphabet{\mathbfsl}{OT1}{ppl}{b}{it}

\usepackage{amsthm}

\newtheorem{theorem}{Theorem}
\newtheorem{lemma}{Lemma}
\newtheorem{remark}{Remark}

\newtheorem{assumption}{Assumption}

\title{Secure Aggregation for Buffered Asynchronous Federated Learning}

\author{%
  Jinhyun So \\
  ECE Department \\
  University of Southern California (USC) \\
  \href{mailto:jinhyuns@usc.edu}{\texttt{jinhyuns@usc.edu}}  \\
  \And
  Ramy E. Ali \\
  ECE Department \\
  University of Southern California (USC) \\
  \href{mailto:reali@usc.edu}{\texttt{reali@usc.edu}}\\
  \AND
  Ba\c{s}ak G\"uler \\
  ECE Department\\
  University of California, Riverside\\
  \href{mailto:bguler@ece.ucr.edu}{\texttt{bguler@ece.ucr.edu}} \\
  \And
  A. Salman Avestimehr \\
  ECE Department \\
  University of Southern California (USC) \\
  \href{mailto:avestime@usc.edu}{\texttt{avestime@usc.edu}} \\
}

\hypersetup{
pdftitle={A template for the arxiv style},
pdfsubject={q-bio.NC, q-bio.QM},
pdfauthor={David S.~Hippocampus, Elias D.~Striatum},
pdfkeywords={First keyword, Second keyword, More},
}

\begin{document}
\maketitle

\begin{abstract}
Federated learning (FL) typically relies on synchronous training, which is slow due to stragglers. While asynchronous training handles stragglers efficiently, it does not ensure privacy due to the incompatibility with the secure aggregation protocols. A buffered asynchronous training protocol known as \FedBuffspace has been proposed recently which bridges the gap between synchronous and asynchronous training to mitigate stragglers and to also ensure privacy simultaneously. \FedBuffspace allows the users to send their updates asynchronously while ensuring privacy by storing the updates in a trusted execution environment (TEE) enabled private buffer. TEEs, however, have limited memory which limits the buffer size. Motivated by this limitation, we develop a buffered asynchronous secure aggregation (\name) protocol that does not rely on TEEs. The conventional secure aggregation protocols cannot be applied in the buffered asynchronous setting since the buffer may have local models corresponding to different rounds and hence the masks that the users use to protect their models  may not cancel out.  \namespace addresses this challenge by carefully designing the masks such that they cancel out even if they correspond to different rounds. Our convergence analysis and experiments show that \namespace almost has the same convergence guarantees as \FedBuffspace without relying on TEEs. 
\end{abstract}


\section{Introduction}\label{sec:intro}
Federated learning (FL) allows users to collaboratively train a machine learning model without sharing their data and while protecting their privacy \cite{mcmahan2016communication}. The training is typically coordinated by a central server. The main idea that enables decentralized training without sharing data is that each user trains a local model using its dataset and the global model maintained by the server. The users then only share their local models with the server which updates the global model and pushes it again to the users for the next training round until convergence. Recent studies, however, showed that sharing the local models still breaches the privacy of the users through inference or inversion attacks  e.g.,~\cite{fredrikson2015model,nasr2019comprehensive,zhu2020deep,geiping2020inverting}. To overcome this challenge, secure aggregation protocols were developed to ensure that the server only learns the global model without revealing the local models \cite{bonawitz2017practical,so2021turbo,kadhe2020fastsecagg,zhao2021information,elkordy2020secure,bell2020secure}. FL protocols commonly rely on synchronous training \cite{mcmahan2016communication}, which  suffers from stragglers due to waiting for the updates of a  sufficient number of users at each round. Asynchronous FL tackles this by incorporating the updates of the users as soon as they arrive at the server \cite{xie2019asynchronous,van2020asynchronous,chai2020fedat,chen2020asynchronous}. While asynchronous FL handles stragglers efficiently, it is not compatible with the secure aggregation protocols designed particularly for synchronous FL. This is because these protocols securely aggregate many local models together each time the global model is updated and hence they are not suitable for asynchronous FL in which each single local model updates the global model. Another approach that can be applied in asynchronous FL to protect the privacy of the users is local differential privacy (LDP) \cite{truex2020ldp}. In this approach, each user adds a noise to the local model before sharing it with the server. This approach, however, degrades the training accuracy.

In \cite{nguyen2021federated}, an asynchronous aggregation protocol known as \FedBuffspace has been proposed to mitigate stragglers and enable secure aggregation jointly. \FedBuffspace enables secure aggregation through trusted execution environments (TEEs) as Intel software guard extensions (SGX) \cite{costan2016intel}. Specifically, the individual updates are not incorporated by the server as soon they arrive. Instead, the server keeps the received local models in a TEE-enabled \emph{secure buffer} of size $K$, where $K$ is a tunable parameter. The server then updates the global model when the buffer is full.
This idea has been shown to be  $3.8$ times faster than the conventional synchronous FL schemes.

\textbf{Contributions}. Since TEEs have limited memory, which limits the buffer size $K$, and are inefficient compared to the untrusted hardware \cite{costan2016intel}, we instead develop a buffered asynchronous secure aggregation protocol that does not rely on TEEs. The main challenge of leveraging the conventional secure aggregation protocols in the buffered asynchronous setting is that the pairwise masks may not cancel out. This is because of the asynchronous nature of this setting which may result in local models of different rounds in the buffer, while the pairwise masks cancel out if they belong to the same round.  This requires a careful design of the masks such that they can be cancelled even if they do not correspond to the same round. Specifically, our contributions are as follows.

\begin{enumerate}[leftmargin=*]
    \item We propose a buffered asynchronous secure aggregation protocol that extends a recently proposed synchronous secure aggregation protocol known as \LightSecAggspace \cite{LightSecAgg2021} to this buffered  asynchronous setting. The key idea of our protocol, \name, is that we design the masks such that they cancel out even if they correspond to different training rounds. 
    \item We extend the convergence analysis of \cite{nguyen2021federated} to the case where the local updates are quantized, which is necessary for the secure aggregation protocols to protect the privacy of the local updates.  
    \item Our extensive experiments on MNIST and CIFAR datasets  show that \namespace almost has the same convergence guarantees as \FedBuffspace despite the quantization. 
\end{enumerate}

\section{Related Works}\label{sec:RelatedWorks}
Secure aggregation protocols typically rely on exchanging pairwise random-seeds and secret sharing them to tolerate users' dropouts \cite{bonawitz2017practical,so2021turbo,kadhe2020fastsecagg,bell2020secure}. The running time of such approaches, however, increases significantly with the number of dropped users since the server needs to reconstruct the mask of each dropped user. Recently, a secure aggregation protocol known as \LightSecAggspace has been proposed to address this challenge \cite{LightSecAgg2021}. In \LightSecAgg, unlike the prior works, the server does not reconstruct the pairwise random-seeds of each  dropped user. Instead, the server directly reconstructs the aggregate masks of all surviving users. This one-shot reconstruction of the masks of all surviving users results in a much faster  training. It is also worth noting that the protocol of  \cite{zhao2021information} is based on the one-shot reconstruction idea, but it requires a trusted third party unlike  \LightSecAgg.

Prior secure aggregation protocols \cite{bonawitz2017practical,so2021turbo,kadhe2020fastsecagg,bell2020secure} are designed for the synchronous FL algorithms such as \FedAvgspace \cite{mcmahan2016communication}, which suffer from stragglers. Asynchronous FL handles this problem by updating the global model as soon as the server receives any local model \cite{xie2019asynchronous,van2020asynchronous,chai2020fedat,chen2020asynchronous}. The larger staleness is of the local model, the greater is the error when updating the global model \cite{xie2019asynchronous}. To address this staleness problem, an asynchronous protocol known as \FedAsyncspace has been developed in \cite{xie2019asynchronous} that updates the global model through staleness-aware weighted averaging of the old global model and the received local model. In \cite{chai2020fedat}, an asynchronous protocol known as \FedAtspace has been proposed, which bridges the gap between synchronous FL and asynchronous FL by developing a semi-synchronous protocol that groups the users, synchronously updates the model of each group and then asynchronously updates the global model across groups. Similarly, a semi-synchronous FL protocol has been developed in \cite{park2020sself} to handle the staleness problem and also mitigates Byzantine users simultaneously.

Asynchronous FL, however, is not compatible with secure aggregation. A potential approach to ensure privacy then is through  DP approaches that add noise to the local models before sharing them with the sever \cite{van2020asynchronous}. A similar approach has been also leveraged in \cite{gu2021privacy} to develop a privacy-preserving protocol for a limited class of learning problems as linear regression, logistic regression and least-squares support vector machine in the vertically partitioned (VP) asynchronous decentralized FL setting. Adding noise, however, degrades the training accuracy. In  \cite{nguyen2021federated}, an asynchronous aggregation protocol known as \FedBuffspace has been proposed to mitigate stragglers while ensuring privacy. The key idea of \FedBuffspace is that the server stores the local models in a TEE-enabled secure buffer of size $K$ until the buffer is full and then securely aggregates them. Due to the memory limitations of TEEs, this approach is only feasible when $K$ is small. This motivates us in this work to develop a buffered asynchronous secure aggregation protocol \emph{without TEEs}. 

\section{Synchronous Secure Aggregation}\label{sec:SyncSecAgg}
In this section, we provide an overview of secure aggregation of synchronous FL. \\ The goal in FL is to collaboratively learn a  global model $\mathbfsl{x}$ with dimension $d$, using the local datasets of $N$ users without sharing them. This problem can be formulated as minimizing a global loss function as follows
\begin{equation}\label{eq:objective_fnc} 
    \min_{\mathbfsl{x} \in \mathbb R^d} 
    F(\mathbfsl{x}) = \sum_{i=1}^N w_i F_i (\mathbfsl{x}), 
\end{equation} 
where $F_i$ is the local loss function of user $i \in [N]$ and $w_i \geq 0$ are the weight parameters that indicate the relative impact of the users and are selected such that $\sum_{i=1}^{N} w_i=1$. \\ This problem is  solved iteratively. At round $t$, the server sends the global model $\mathbfsl{x}^{(t)}$ to the users. Some of the users may dropout due to various reasons such as wireless connectivity. We assume that at most $D$ users may dropout in any round.  We denote the set of the surviving users at round $t$ by $\mathcal U^{(t)}$ and the set of dropped users by $\mathcal D^{(t)}$. User $i\in[N]$ updates the global model by carrying out $E \geq1$ local stochastic gradient descent (SGD) steps. The goal of the server is to get the sum of the local models of the surviving users to update its global model as $\mathbfsl x^{(t)}=\frac{1}{|\mathcal U^{(t)}|} \sum_{i \in \mathcal U^{(t)}} \mathbfsl x_i ^{(t)}.$ The server then sends $\mathbfsl x^{(t)}$ to the users for the next round. While the users do not share their data with the server and just share their local models, the local models still reveal significant information about their datasets \cite{fredrikson2015model,nasr2019comprehensive,zhu2020deep,geiping2020inverting}. To address this challenge, a secure aggregation protocol known as \secaggspace was developed in \cite{bonawitz2016practical} to ensure that the server does not learn anything about the local models except $\sum_{i \in \mathcal U^{(t)}} \mathbfsl x_i ^{(t)}$ at round $t$. Specifically, we assume that up to $T$ users can collude with each other as well as with the server to reveal the local models of other users. The secure aggregation protocol then must ensure that nothing is revealed beyond the aggregate model despite such collusions. 

\subsection{Overview of \secaggspace}\label{subsec:secagg}
We now provide an overview of \secagg. In this discussion, we omit the round index $t$ for simplicity since the procedure is the same at each round. \secaggspace ensures privacy against any subset of up to $T$ colluding users and resiliency against $D$ colluding workers provided that $N>D+T$. \\ In \secagg, the users mask their models before sharing them with the server using random keys. Specifically, each pair of users $i, j \in [N]$ agree on a pairwise random seed $a_{i, j}$. Moreover, user $i$ also uses a private random seed $b_i$ that is used when the update of this user is delayed but eventually reaches the server. The model of user $i$ is then masked as follows
\begin{align}\label{eq:secagg_maskedmodel}
    \mathbfsl y_i=\mathbfsl x_i+\mathrm{PRG}(b_i)+\sum_{j: i<j} \mathrm{PRG}(a_{i, j})-\sum_{j: i>j} \mathrm{PRG}(a_{j, i}),
\end{align}
where $\mathrm{PRG}$ is a pseudo random generator. The server then reconstructs the private random-seed of each surviving user, the pairwise random-seed of each dropped user and recovers the aggregate model of the surviving users as follows 
\begin{align}
    \sum_{i \in \mathcal U} \mathbfsl x_i= \sum_{i \in \mathcal U}  (\mathbfsl y_i-\mathrm{PRG}(b_i))+ \sum_{i \in \mathcal D} \left(\sum_{j: i<j} \mathrm{PRG}(a_{i, j})-\sum_{j: i>j} \mathrm{PRG}(a_{j, i})\right).
\end{align}

\subsection{Overview of \LightSecAggspace}\label{subsec:lightsecagg}
Next, we provide an overview of \LightSecAgg. \LightSecAggspace has three parameters $T$ that represents the privacy guarantee, $D$ that represents that dropout guarantee and $U$ which represents the targeted number of surviving users. These parameters must be selected such that $N-D \geq U \geq T$. In \LightSecAgg, user $i$ selects a random mask $\mathbfsl z_i$ and partitions it to $U-T$ sub-masks denoted by $[\mathbfsl z_i]_1, \cdots, [\mathbfsl z_i]_{U-T}$. User $i$ also selects another $T$ random masks denoted by $[\mathbfsl n_i]_{U-T+1}, \cdots, [\mathbfsl n_i]_{U}$. These $U$ partitions $[\mathbfsl z_i]_1, \cdots, [\mathbfsl z_i]_{U-T}, [\mathbfsl n_i]_{U-T+1}, \cdots, [\mathbfsl n_i]_{U}$ are then encoded through an $(N, U)$ Maximum Distance Separable (MDS) code \cite{macwilliams1977theory} as follows 
\begin{align}
   [\widetilde{\mathbfsl z}_i]_j= \left([\mathbfsl z_i]_1, \cdots, [\mathbfsl z_i]_{U-T}, [\mathbfsl n_i]_{U-T+1}, \cdots, [\mathbfsl n_i]_{U} \right) \mathbfsl v_j,
\end{align}
where $\mathbfsl v_j$ is the $j$-th column of a Vandermonde matrix $\mathbf V \in \mathbb F_q^{U \times N}$. After that, user $i$ sends $[\widetilde{\mathbfsl z}_i]_j$ to user $j \in [N] \setminus \{i\}$. User $i$ then masks its model as
$\mathbfsl y_i=\mathbfsl x_i+\mathbfsl z_i.$

The goal of the server now is to recover the aggregate model $\sum_{i \in \mathcal U_1} \mathbfsl x_i$, where $\mathcal U_1$ is the set of surviving users in this phase.  To do so, each surviving users $j \in \mathcal U_1$ sends $\sum_{i \in \mathcal U_1} [\widetilde{\mathbfsl z}_i]_j$ to the server. The server then directly recovers $\sum_{i \in \mathcal U_1} [\mathbfsl z_i]_k$ for $k \in [U-T]$ through MDS decoding when it receives at least $U$ messages from the surviving users. We denote this subset of the surviving users by $\mathcal U_2$, where $|\mathcal U_2|=U$. Finally, the server recovers the aggregate model as $\sum_{i \in \mathcal U_1} \mathbfsl x_i=\sum_{i \in \mathcal U_1} \mathbfsl y_i-\sum_{i \in \mathcal U_1} \mathbfsl z_i$.  

\section{Buffered Asynchronous Secure Aggregation}\label{sec:BASecAgg}
In this section, we provide a brief overview of \FedBuffspace \cite{nguyen2021federated}. Then, we illustrate the incompatibility of the conventional secure aggregation with asynchronous FL in Section \ref{subsec:incompatibility}. Later on, in Section \ref{subsec:proposed}, we introduce \name. 


In asynchronous FL, the updates of the users are not synchronized while the goal is the same as the synchronous FL to minimize the global loss function in \eqref{eq:objective_fnc}. In the buffered asynchronous setting, the server stores each local model that it receives in a buffer of size $K$ and updates the global model when the buffer is full. In our setting, this buffer is \emph{not a secure buffer}. Hence, our goal is to design the secure aggregation protocol where users send the masked updates to protect the privacy in a way that the server can aggregate the local updates while the server (and potential colluding users) learns no information about the local updates beyond the aggregate of the updates stored in the buffer.

\textbf{FedBuff.} Before presenting our protocol, \name, we first provide an overview about the buffered asynchronous aggregation framework, named \FedBuff~\cite{nguyen2021federated}, and describe the challenges that render \secaggspace incompatible with this framework. The key intuition of \FedBuffspace is to introduce a new design parameter $K$, the buffer size at the server, so that \FedBuffspace has two degrees of freedom, $K$ and the concurrency $C$  while the synchronous FL frameworks have only one degree of freedom, concurrency. The concurrency is the number of users training concurrently and is an important parameter to provide a trade-off between the training time and the data inefficiency. Synchronous FL speeds up the training by increasing the concurrency, but higher concurrency results in data inefficiency~\cite{nguyen2021federated}. In \FedBuff, however, a high concurrency coupled with a proper value of $K$ results in fast training. In other words, the additional degree of freedom $K$ allows the server to update more frequently than concurrency, which enables  \FedBuffspace to achieve data efficiency at high concurrency. 

At round $t$, $C$ users are locally training the model by carrying out $E\geq1$ local SGD steps.
When the local update is done, user $i$ sends the difference between the downloaded global model and updated local model to the server. 
The local update of user $i$ sent to the server at round $t$ is given by
\begin{equation}\label{eq:local_update}
    {\Delta}^{(t;t_i)}_i = \mathbfsl{x}^{(t_i)} - \mathbfsl{x}^{(E)}_i,
\end{equation}
where $t_i$ is the latest round index when the global model is downloaded by user $i$ and $t$ is the round index when the local update is sent to the server, hence the staleness of user $i$ is given by $\tau_i = t- t_i$. $\mathbfsl{x}^{(E)}_i$ denotes the local model after $E$ local SGD steps and the local model at user $i$ is updated as  
\begin{equation}\label{eq:onestep_localSGD}
    \mathbfsl{x}^{(e)}_i = \mathbfsl{x}^{(e-1)}_i - \eta_l g_i(\mathbfsl{x}^{(e-1)}_i;\xi_i)
\end{equation}
for $e=1,\ldots,E$, where $\mathbfsl{x}^{(0)}_i = \mathbfsl{x}^{(t_i)}$, $\eta_l$ denotes learning rate of the local updates.
$g_i(\mathbfsl{x};\xi_i)$ denotes the stochastic gradient with respect to the random sampling $\xi_i$ on user $i$, and we assume $\mathbb{E}_{\xi_i}[g_i(\mathbfsl{x};\xi_i)] = \nabla F_i(\mathbfsl{x})$ for all $\mathbfsl{x}\in\mathbb{R}^d$ where $F_i$ is the local loss function of user $i$ defined in \eqref{eq:objective_fnc}. 
The server stores the received local updates in a buffer of size $K$. When the buffer is full, the server updates the global model by subtracting the aggregate of all local updates from the current global model. Specifically, the global model at the server is updated as  
\begin{equation}\label{eq:global_update}
    \mathbfsl{x}^{(t+1)} = \mathbfsl{x}^{(t)} - \frac{\eta_g}{\sum_{i\in\mathcal{S}^{(t)}} s(t-t_i)} \sum_{i\in \mathcal{S}^{(t)}} s(t-t_i){\Delta}^{(t;t_i)}_i,
\end{equation}
where $\mathcal{S}^{(t)}$ is an index set of the $K$ users whose local models are in the buffer at round $t$ and $\eta_g$ is the learning rate of the global updates. $s(\tau)$ is a function that compensates for the staleness satisfying $s(0)=1$ and is monotonically decreasing as $\tau$ increases. 
There are many functions that satisfy these two properties and we consider a polynomial function $s_{\alpha}(\tau) = (\tau +1)^{-\alpha}$ as it shows similar or better performance than the other functions e.g., Hinge or Constant stale function~\cite{xie2019asynchronous}.

\textbf{Privacy and Dropout Model.} We assume at most $D$ users may dropout in any round and a threat model where the users and the server are honest but curious who follow the protocol but try to infer the local updates of the other users. 
The secure aggregation protocol guarantees that nothing beyond the aggregate of the local updates is revealed, even if up to $T$ users collude with the server. 
We consider information-theoretic privacy where from every subset of users $\mathcal{T} \subseteq [N]$ of size at most $T$, we must have mutual information $I(\{{\Delta}^{(t;t_i)}_i\}_{i \in [N]};\mathbf{Y}^{(t)}|\sum_{i \in \mathcal{S}^{(t)}}{\Delta}^{(t;t_i)}_i, \mathbf{Z}^{(t)}_{\mathcal{T}}) = 0$, where $\mathbf{Y}^{(t)}$ and $\mathbf{Z}^{(t)}_{\mathcal{T}}$ are the collection of information at the server and at the users in $\mathcal{T}$ at round $t$, respectively.


\subsection{Incompatibility  of \secaggspace with  Buffered Asynchronous FL}
\label{subsec:incompatibility}

As described in Section~\ref{subsec:secagg}, \secaggspace \cite{bonawitz2017practical} is designed for synchronous FL. At round $t$, each pair of users $i,j\in[N]$ agree on a pairwise random-seed $a_{i,j}^{(t)}$, and generate a random vector by running PRG based on the random seed of $a_{i,j}^{(t)}$ to mask the local update. 
This additive structure has the unique property that these pairwise random vectors cancel out when the server aggregates the masked models because user $i(<j)$ adds $\mathrm{PRG}(a_{i,j}^{(t)})$ to $\mathbfsl{x}_i^{(t)}$ and user $j(>i)$ subtracts $\mathrm{PRG}(a_{i,j}^{(t)})$ from $\mathbfsl{x}_j^{(t)}$.

In the buffered asynchronous FL, however, the cancellation of the pairwise random masks based on the key agreement protocol is not guaranteed due to the mismatch in staleness between users. Specifically, at round $t$, user $i\in\mathcal{S}^{(t)}$ sends the masked model $\mathbfsl{y}_i^{(t;t_i)}$ to the server that is given by
\begin{equation} \small
    \mathbfsl{y}_i^{(t;t_i)} = {\Delta}^{(t;t_i)}_i + \mathrm{PRG}\left(b_i^{(t_i)}\right)+\sum_{j: i<j} \mathrm{PRG}\left(a^{(t_i)}_{i, j}\right)-\sum_{j: i>j} \mathrm{PRG}\left(a^{(t_i)}_{j, i}\right),
\end{equation}
where ${\Delta}^{(t;t_i)}_i$ is the local update defined in \eqref{eq:local_update}.
When $t_i \neq t_j$, the pairwise random vectors in $\mathbfsl{y}_i^{(t;t_i)}$ and $\mathbfsl{y}_j^{(t;t_j)}$ are not canceled out as $a^{(t_i)}_{i, j} \neq a^{(t_j)}_{i, j}$.
We note that the identity of the staleness of each user is not known a priori, hence each pair of users cannot use the same pairwise random-seed.

\subsection{The Proposed \namespace Protocol}
\label{subsec:proposed}

To address the challenge of asynchrony in the buffered asynchronous secure aggregation, we propose \namespace by modifying the idea of \emph{one-shot} recovery leveraged in \LightSecAggspace \cite{LightSecAgg2021} to our setting. We provide a brief overview of \LightSecAggspace in Section \ref{subsec:lightsecagg}. Our key intuition is to encode the local masks in a way that the server can recover the aggregate of masks from the encoded masks via a one-shot computation even though the masks are generated in different training rounds.

\namespace has three phases. 
First, each user generates a random mask to protect the privacy of the local update, and further creates encoded masks via a  \emph{$T$-private} Maximum Distance Separable (MDS) code that provides privacy against $T$ colluding users. 
Each user sends one of the encoded masks to one of the other users for the purpose of one-shot recovery.
Second, each user trains a local model and converts it from the domain of real numbers to the finite field as generating random masks and MDS encoding are required to be carried out in the finite field to provide information-theoretic privacy. Then, the quantized model is masked by the random mask generated in the first phase, and sent to the server. The server stores the masked update in the buffer.  
Third, when the buffer is full, the server aggregates the $K$ masked updates in the buffer. 
To remove the randomness in the aggregate of the masked updates, the server reconstructs the aggregated masks of the users in the buffer. To do so, each surviving user sends the aggregate of encoded masks to the server. After receiving a sufficient number of aggregated encoded masks, the server reconstructs the aggregate of masks and hence the aggregate of the $K$ local updates. 
We now describe these three phases in detail. 

\subsubsection{Offline Encoding and Sharing of Local Masks}\label{subsubsec:BASecAgg_firstphase} 
User $i$ generates $\mathbfsl{z}_i^{(t_i)}$ uniformly at random from the finite field $\mathbb{F}^d_q$, where $t_i$ is the global round index when user $i$ downloads the global model from the server. The mask $\mathbfsl{z}_i^{(t_i)}$ is partitioned into $U-T$ sub-masks denoted by 
$[\mathbfsl z^{(t_i)}_i]_1, \cdots, [\mathbfsl z^{(t_i)}_i]_{U-T}$, where $U$ denotes the targeted number of surviving users and $N-D\geq U \geq T$. User $i$ also selects another $T$ random masks denoted by $[\mathbfsl n^{(t_i)}_i]_{U-T+1}, \cdots, [\mathbfsl n^{(t_i)}_i]_{U}$. These $U$ partitions $[\mathbfsl z^{(t_i)}_i]_1, \cdots, [\mathbfsl z^{(t_i)}_i]_{U-T}, [\mathbfsl n^{(t_i)}_i]_{U-T+1}, \cdots, [\mathbfsl n^{(t_i)}_i]_{U}$ are then encoded through an $(N, U)$ Maximum Distance Separable (MDS) code as follows 
\begin{align}\label{eq:encoding}
   [\widetilde{\mathbfsl z}^{(t_i)}_i]_j= \left([\mathbfsl z_i^{(t_i)}]_1, \cdots, [\mathbfsl z^{(t_i)}_i]_{U-T}, [\mathbfsl{n}^{(t_i)}_i]_{U-T+1}, \cdots, [\mathbfsl{n}^{(t_i)}_i]_{U} \right) \mathbfsl v_j,
\end{align}
where $\mathbfsl v_j$ is the $j$-th column of a Vandermonde matrix $\mathbf V \in \mathbb F_q^{U \times N}$. After that, user $i$ sends $[\widetilde{\mathbfsl z}^{(t_i)}_i]_j$ to user $j \in [N] \setminus \{i\}$.
At the end of this phase, each user $i\in[N]$ has $[\widetilde{\mathbfsl z}^{(t_j)}_j]_i$ from $j\in[N]$.


\subsubsection{Training, Quantizing, Masking, and Uploading of Local Updates}\label{subsubsec:BASecAgg_secondphase}
Each user $i$ trains the local model as in \eqref{eq:local_update} and \eqref{eq:onestep_localSGD}. 
User $i$ quantizes its local update ${\Delta}^{(t;t_i)}_i$ from the domain of real numbers to the finite field $\mathbb{F}_q$ as masking and MDS encoding are carried out in the finite field to provide information-theoretic privacy.
The field size $q$ is assumed to be large enough to avoid any wrap-around during secure aggregation.

The quantization is a challenging task as it should be performed in a way to ensure the convergence of the global model. Moreover, the quantization should allow the representation of negative integers in the finite field, and enable computations to be carried out in the quantized domain. Therefore, we cannot utilize well-known gradient quantization techniques such as in \cite{alistarh2017qsgd}, which represents the sign of a negative number separately from its magnitude.  
\namespace addresses this challenge with a simple stochastic quantization strategy combined with the two's complement representation as described next. 
For any positive integer $c\geq1$, we first define a stochastic rounding function as
\begin{equation}\label{eq:sto_round}
    Q_c(x) = 
    \left\{
    \begin{array}{ll}
          \frac{\lfloor cx \rfloor}{c}   & \text{with prob. } 1-(cx-\lfloor cx \rfloor)\\
          \frac{\lfloor cx \rfloor+1}{c} & \text{with prob. } cx-\lfloor cx \rfloor,
    \end{array} 
    \right. 
\end{equation}
where $\floor{x}$ is the largest integer less than or equal to $x$, and this rounding function is unbiased, i.e., $\mathbb{E}_Q[Q_c(x)]=x$. The parameter $c$ is a design parameter to determine the number of quantization levels. The variance of $Q_c(x)$ decreases as the value of $c$ increases, which will be described in Lemma~\ref{lemma:q_prop} in Appendix~\ref{app:convergence_proof} in detail.
We then define the quantized update
\begin{equation}\label{eq:def_w_bar}
    \overline{\Delta}^{(t;t_i)}_i := \phi\left({c_l}\cdot Q_{c_l}\left({\Delta}^{(t;t_i)}_i\right)\right),
\end{equation}
where the function $Q_c$ from~\eqref{eq:sto_round} is carried out element-wise, and $c_l$ is a positive integer parameter to determine the quantization level of the local updates.   
The mapping function $\phi:\mathbb{R}\rightarrow\mathbb{F}_q$ is defined to represent a negative integer in the finite field by using the two's complement representation, 
\begin{equation}\label{eq:phi} 
    \phi(x) =
    \left\{
    \begin{array}{ll}
          x & \text{if } x \geq 0\\
          q+x & \text{if } x<0.
    \end{array} 
    \right. 
\end{equation}
To protect the privacy of the local updates, user $i$ masks the quantized update $\overline{\Delta}^{(t;t_i)}_i$ in \eqref{eq:def_w_bar} as
\begin{equation}\label{eq:masked_update}
    \widetilde{\Delta}^{(t;t_i)}_i = \overline{\Delta}^{(t;t_i)}_i + \mathbfsl{z}_i^{(t_i)},
\end{equation}
and sends the pair of $\left\{\widetilde{\Delta}^{(t;t_i)}_i, t_i \right\}$ to the server.
The local round index $t_i$ will be used in two cases: (1) when the server identifies the staleness of each local update and compensates it, and (2) when the users aggregate the encoded masks for one-shot recovery, which will be explained in Section \ref{subsubsec:BASecAgg_thirdphase}.

\subsubsection{One-shot Aggregate-update Recovery and Global Model Update}\label{subsubsec:BASecAgg_thirdphase} 
The server stores $\widetilde{\Delta}^{(t;t_i)}_i$ in the buffer, and when the buffer of size $K$ is full the server aggregates the $K$ masked local updates. In this phase, the server intends to recover 
\begin{equation}
    \sum_{i\in\mathcal{S}^{(t)}} s(t-t_i) {\Delta}^{(t;t_i)}_i,    
\end{equation}
where ${\Delta}^{(t;t_i)}_i$ is the local update in the real domain defined in \eqref{eq:local_update}, $\mathcal{S}^{(t)}$ ($\left| \mathcal{S}^{(t)} \right|=K$) is the index set of users whose local updates are stored in the buffer and aggregated by the server at round $t$, 
and $s(\tau)$ is the staleness function defined in \eqref{eq:global_update}.
To do so, the first step is to reconstruct $\sum_{i\in\mathcal{S}^{(t)}} s(t-t_i)\mathbfsl{z}_i^{(t_i)}$. This is challenging as the decoding should be performed in the finite field, but the value of $s(\tau)$ is a real number.
To address this problem, we introduce a quantized staleness function $\overline{s}:\{0,1,\ldots,\}\rightarrow \mathbb{F}_q$,
\begin{equation}\label{eq:quantized_stale_function}
    \overline{s}_{c_g}(\tau) = c_g Q_{c_g}\left( s(\tau) \right),
\end{equation}
where $Q_c(\cdot)$ is a stochastic rounding function defined in \eqref{eq:sto_round}, and $c_g$ is a positive integer to determine the quantization level of staleness function. 
Then, the server broadcasts information of $\left\{ \mathcal{S}^{(t)}, \left\{ t_i\right\}_{i\in\mathcal{S}^{(t)}}, c_g \right\}$ to all surviving users. 
After identifying the selected users in $\mathcal{S}^{(t)}$, the local round indices $\{t_i\}_{i\in\mathcal{S}^{(t)}}$ and the corresponding staleness, user $j\in[N]$ aggregates its encoded sub-masks $\sum_{i\in\mathcal{S}^{(t)}} \overline{s}_{c_g}(t-t_i) \left[\widetilde{\mathbfsl z}^{(t_i)}_i\right]_j$
and sends it to the server for the purpose of one-shot recovery. The key difference between \namespace and \LightSecAggspace is that in \name, the time stamp $t_i$ for encoded masks $\left[\widetilde{\mathbfsl z}^{(t_i)}_i\right]_j$ for each $i\in\mathcal{S}^{(t)}$ can be different, hence user $j\in[N]$ must aggregate the encoded mask with the proper round index.
Due to the commutative property of coding and linear operations, each $\sum_{i\in\mathcal{S}^{(t)}} \overline{s}_{c_g}(t-t_i) \left[\widetilde{\mathbfsl z}^{(t_i)}_i\right]_j$ is an encoded version of $\sum_{i\in\mathcal{S}^{(t)}} \overline{s}_{c_g}(t-t_i) \left[{\mathbfsl z}^{(t_i)}_i\right]_k$ for $k\in[U-T]$ using the MDS matrix (or Vandermonde matrix) $\mathbf V$ defined in \eqref{eq:encoding}.
Thus, after receiving a set of any $U$ results from surviving users in $\mathcal{U}_2$, where $|\mathcal{U}_2|=U$, the server reconstructs $\sum_{i\in\mathcal{S}^{(t)}} \overline{s}_{c_g}(t-t_i) \left[{\mathbfsl z}^{(t_i)}_i\right]_k$ for $k\in[U-T]$ via MDS decoding.
By concatenating the $U-T$ aggregated sub-masks $\sum_{i\in\mathcal{S}^{(t)}} \overline{s}_{c_g}(t-t_i) \left[{\mathbfsl z}^{(t_i)}_i\right]_k$, the server can recover $\sum_{i\in\mathcal{S}^{(t)}} \overline{s}_{c_g}(t-t_i) {\mathbfsl z}^{(t_i)}_i$.
Finally, the server obtains the desired global update as follows
\begin{equation}
    \mathbfsl{g}^{(t)} =
    \frac{1}{c_g c_l \sum_{i\in\mathcal{S}^{(t)}} {s}_{c_g}(t-t_i)} \phi^{-1} \left( \sum_{i\in\mathcal{S}^{(t)}} \overline{s}_{c_g}(t-t_i)\widetilde{{\Delta}}^{(t;t_i)}_i - \sum_{i\in\mathcal{S}^{(t)}} \overline{s}_{c_g}(t-t_i) {\mathbfsl z}^{(t_i)}_i\right),
\end{equation}
where $c_l$ is defined in \eqref{eq:def_w_bar} and $\phi^{-1}:\mathbb{F}_q\rightarrow\mathbb{R}$ is the demapping function defined as follows
\begin{equation}\label{eq:inv_phi}
    {\phi}^{-1}(\overline{x})=
    \left\{
    \begin{array}{ll}
          \overline{x} & \text{if \quad } 0 \leq \overline{x} < \frac{q-1}{2}\\
          \overline{x}-q & \text{if \quad } \frac{q-1}{2} \leq \overline{x} < q.
    \end{array}
    \right.
\end{equation}
Finally, the server updates the global model as $\mathbfsl{x}^{(t+1)} = \mathbfsl{x}^{(t)} - \eta_g \mathbfsl{g}^{(t)}$, which is equivalent to
\begin{equation}\label{eq:global_update_equivalent}
    \mathbfsl{x}^{(t+1)} = \mathbfsl{x}^{(t)} - \frac{\eta_g}{\sum_{i\in\mathcal{S}^{(t)}} Q_{c_g}\left(s(t-t_i) \right)} \sum_{i\in \mathcal{S}^{(t)}} Q_{c_g}\left(s(t-t_i) \right) Q_{c_l}\left({\Delta}^{(t;t_i)}_i\right),
\end{equation}
where $Q_{c_l}$ and $Q_{c_g}$ are the stochastic rounding function defined in \eqref{eq:sto_round} with respect to quantization parameters $c_l$ and $c_g$, respectively.

\section{Convergence Analysis}\label{sec:Convergence}
In this section, we provide the convergence guarantee of \namespace in the $L$-smooth and non-convex setting. 
For simplicity, we consider the constant staleness function $s(\tau)=1$ for all $\tau$ in \eqref{eq:global_update_equivalent}. Then, the global update equation of \namespace is given by
\begin{equation}\label{eq:global_update_conv}
    \mathbfsl{x}^{(t+1)} = \mathbfsl{x}^{(t)} - \frac{\eta_g}{K} \sum_{i\in \mathcal{S}^{(t)}}  Q_{c_l}\left({\Delta}^{(t;t_i)}_i\right),
\end{equation}
where $Q_{c_l}$ is the stochastic round function defined in \eqref{eq:sto_round}, $c_l$ is the positive constant to determine the quantization level, and ${\Delta}^{(t;t_i)}_i$ is the local update of user $i$ defined in \eqref{eq:local_update}.
We first introduce our assumptions, which are commonly made in analyzing FL algorithms \cite{li2019convergence, nguyen2021federated, reddi2020adaptive, so2021securing}.

\begin{assumption} \label{assumpt:1} (Unbiasedness of local SGD).
For all $i\in[N]$ and $\mathbfsl{x}\in\mathbb{R}^d$, $\mathbb{E}_{\xi_i}[g_i(\mathbfsl{x};\xi_i)] = \nabla F_i(\mathbfsl{x})$ where $g_i(\mathbfsl{x};\xi_i)$ is the stochastic gradient estimator of user $i$ defined in \eqref{eq:onestep_localSGD}.
\end{assumption}

\begin{assumption} \label{assumpt:2} (Lipschitz gradient). 
$F_1,\ldots,F_N$ in \eqref{eq:objective_fnc} are all $L$-smooth: for all $\mathbfsl{a},\mathbfsl{b}\in\mathbb{R}^d$ and $i\in[N]$, 
$\lVert \nabla F_i(\mathbfsl{a}) - \nabla F_i(\mathbfsl{b}) \rVert^2 \leq L \lVert \mathbfsl{a}-\mathbfsl{b} \rVert^2$.
\end{assumption}

\begin{assumption}\label{assumpt:3} (Bounded variance of local and global gradients). 
The variance of the stochastic gradients at each user is bounded, i.e., $\mathbb{E}_{\xi_i} \left\lVert \nabla g_i(\mathbfsl{x};\xi_i) - \nabla F_i(\mathbfsl{x}) \right\rVert^2 \leq \sigma^2_l$ for $i\in [N]$ and $\mathbfsl{x}\in\mathbb{R}^d$.
For the global loss function $F(\mathbfsl{x})$ defined in \eqref{eq:objective_fnc}, $\frac{1}{N}\sum_{i=1}^{N} \left\lVert \nabla F_i(\mathbfsl{x}) - \nabla F(\mathbfsl{x}) \right\rVert^2 \leq \sigma^2_{g}$ holds.
\end{assumption}

\begin{assumption}\label{assumpt:4}
(Bounded gradient). For all $i\in [N]$, $\lVert\nabla F_i(\mathbfsl{x})\rVert^2 \leq G$.
\end{assumption}

In addition, we make an assumption on the staleness of the local updates under asynchrony~\cite{nguyen2021federated}.
\begin{assumption}\label{assumpt:5}(Bounded staleness).
For each global round index $t$ and all users $i\in[N]$, the delay $\tau_i^{(t)} = t - t_i$ is not larger than a certain threshold $\tau_{\mathrm{max}}$ where $t_i$ is the latest round index when the global model is downloaded to user $i$.
\end{assumption}


Now, we state our main result for the convergence guarantee of \name. 

\begin{theorem}\label{thm:convergence}
Selecting the constant learning rates $\eta_l$ and $\eta_g$ such that $\eta_l \eta_g K E \leq \frac{1}{L}$, the global model iterates in \eqref{eq:global_update_conv} achieve the following ergodic convergence rate
\begin{equation}\label{eq:conv_rate}
    \frac{1}{J} \sum_{t=0}^{J-1} \mathbb{E} \left[ \lvert \nabla F(\mathbfsl{x}^{(t)}) \rvert^2 \right]
    \leq \frac{2F^{*}}{\eta_g \eta_l E K T} + \frac{L \eta_g \eta_l \sigma^2_{c_l}}{2}
    + 3L^2 E^2 \eta_l^2 \left( \eta_g^2 K^2 \tau^2_{\mathrm{max}} \right) \sigma^2,
\end{equation}
where $F^{*}=F(\mathbfsl{x}^{(0)}) - F(\mathbfsl{x}^{*})$, $\sigma^2 = G + \sigma_g^2 + \sigma_{c_l}^2$, and $\sigma_{c_l}^2 = \frac{d}{4{c_l}^2} + \sigma^2_l$.
\end{theorem}

\noindent The proof of Theorem \ref{thm:convergence} is provided in Appendix \ref{app:convergence_proof}.

\begin{remark} \normalfont \label{remark:quantization_var}
Theorem \ref{thm:convergence} shows that convergence rates of \namespace and \FedBuffspace (see Corollary 1 in \cite{nguyen2021federated}) are the same except for the increased variance of the local updates due to the quantization noise in \name. 
The amount of the increased variance $\frac{d}{4{c_l}^2}$ in $\sigma_{c_l}^2 = \frac{d}{4{c_l}^2} + \sigma^2_l$ is negligible for large ${c_l}$, which will be demonstrated in our experiments in Section \ref{sec:Experiments}.
\end{remark}

\section{Experiments}\label{sec:Experiments}
    
    
    
    

In this section, we demonstrate the convergence performance of \namespace compared to the buffered asynchronous FL scheme from \cite{nguyen2021federated} termed \FedBuff. We measure the performance in terms of the model accuracy evaluated over the test samples with respect to the global round index $t$.

\noindent {\bf Datasets and network architectures.} We consider an image classification task on the MNIST dataset \cite{lecun1998mnist} and CIFAR-10 dataset \cite{krizhevsky2009learning}.
For MNIST dataset, we train LeNet \cite{lecun1998mnist}. For CIFAR-10 dataset, we train the convolutional neural network (CNN) used in \cite{xie2019asynchronous}. 
These network architectures are sufficient for our needs as our goal is to evaluate various schemes, not to achieve the best accuracy.
More details about hyperparameters are provided in Appendix \ref{app:exp_details}.

\noindent {\bf Setup.} We consider a buffered asynchronous FL setting with $N=100$ users and a single server having the buffer of size $K=10$. For IID data distribution, the training samples are shuffled and partitioned into $N=100$ users. 
For asynchronous training, we assume the staleness of each user is uniformly distributed over $[0,10]$, i.e., $\tau_{\mathrm{max}}=10$, as used in \cite{xie2019asynchronous}.
We set the field size $q=2^{32}-5$, which is the largest prime within $32$ bits. 
 
\noindent {\bf Implementations.} 
We implement two schemes, \FedBuffspace and \name. 
The key difference between two schemes is that in \name, the local updates are quantized and converted into the finite field to provide privacy of the individual local updates while all operations are carried out over the domain of real numbers in \FedBuff.
For both schemes, to compensate the staleness of the local updates, we employ the two strategies for the weighting function: a constant function $s(\tau)=1$ and a polynomial function $s_\alpha(\tau)=(1+\tau)^{-\alpha}$. 

\noindent {\bf Empirical results.}
In Figure \ref{fig:CNN_MNIST} and \ref{fig:CNN_CIFAR10}, we demonstrate that \namespace has almost the same performance as \FedBuffspace on both MNIST and CIFAR-10 datasets while \namespace includes quantization noise to protect the privacy of individual local updates of users. 
This is because the quantization noise in \namespace is negligible as explained in Remark \ref{remark:quantization_var}.
To compensate the staleness of the local updates over the finite field in \name, we implement the quantized staleness function defined in \eqref{eq:quantized_stale_function} with $c_g=2^6$, which has the same performance to mitigate the staleness as the original staleness function carried out over the domain of real numbers. 

\begin{figure*}[t!]
\centering
    \subfigure[MNIST dataset.]{\label{fig:CNN_MNIST}
    \includegraphics[scale=0.45]{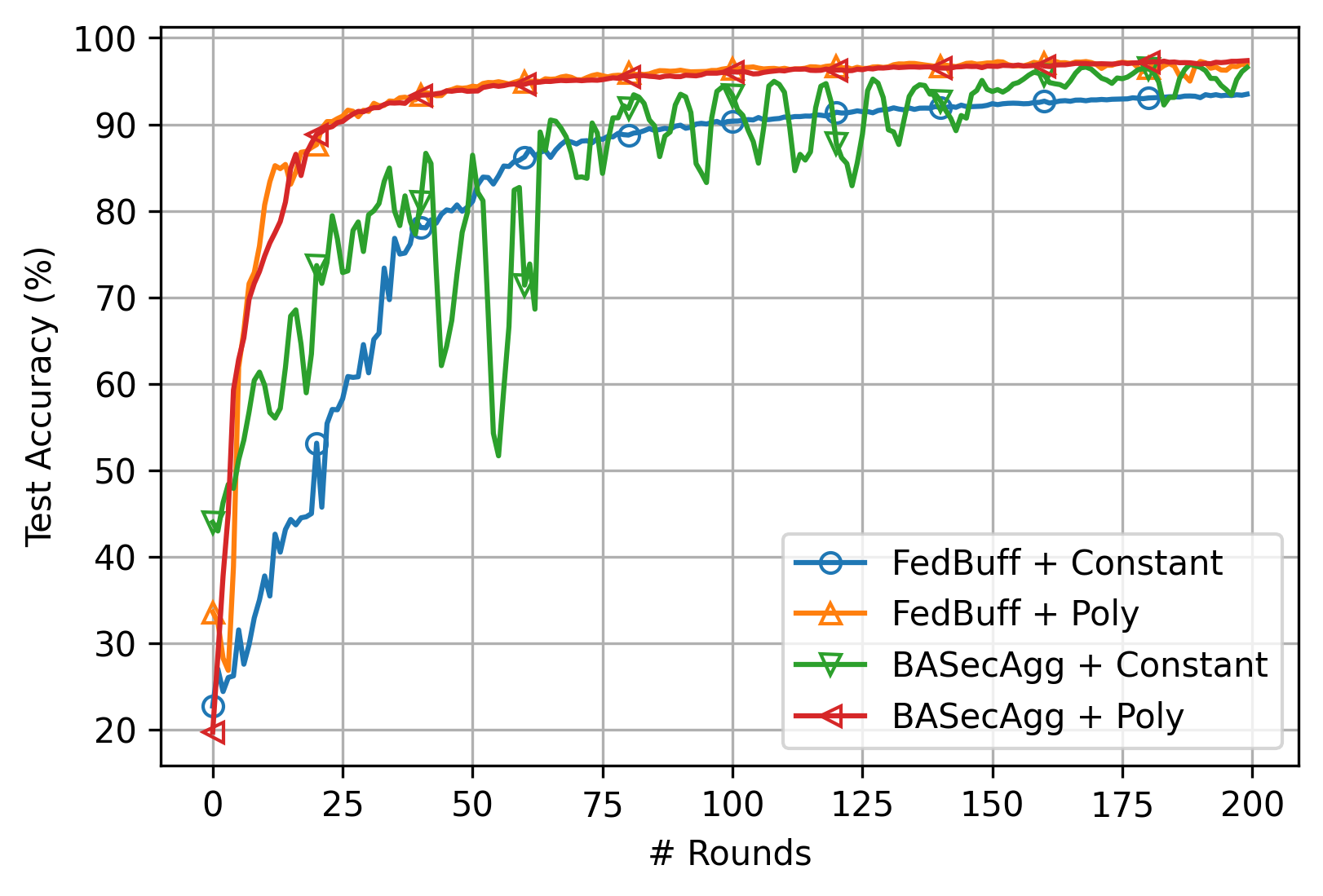}
    }
    \subfigure[CIFAR-$10$ dataset.]{\label{fig:CNN_CIFAR10}
    \includegraphics[scale=0.45]{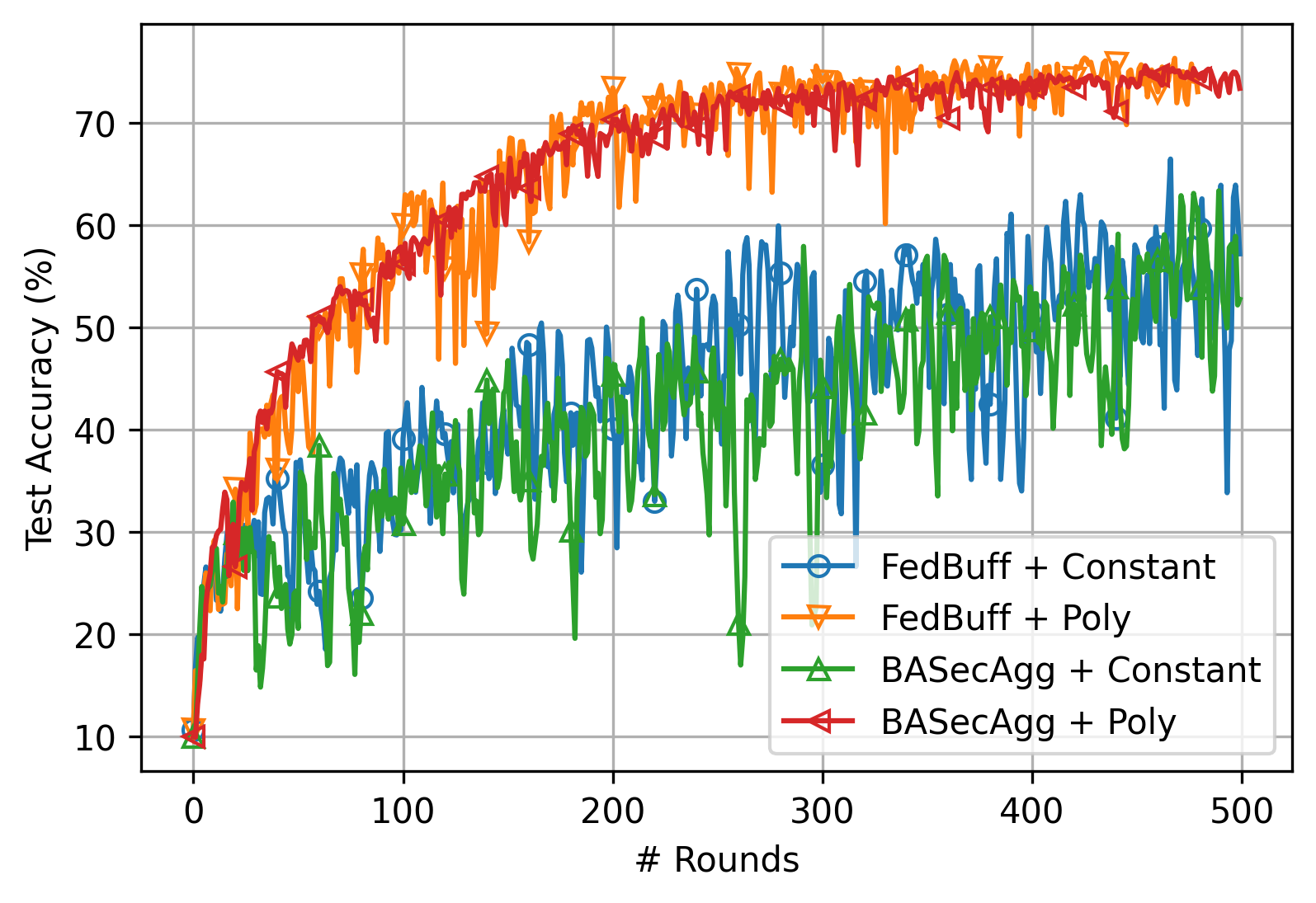}
    }
\vspace{-5 pt}
\caption{\footnotesize Accuracy of \namespace and \FedBuffspace with two strategies for the weighting function to mitigate the staleness: a constant function $s(\tau)=1$ (no compensation) named \emph{Constant}; and a polynomial function $s_\alpha(\tau)=(1+\tau)^{-\alpha}$ named \emph{Poly} where $\alpha=1$.}
\label{fig:accuracy}
\vspace{-0.15cm}
\end{figure*}

\begin{figure*}[t!]
\centering
    \subfigure[MNIST dataset.]{\label{fig:MNIST_diff_q_level}
    \includegraphics[scale=0.45]{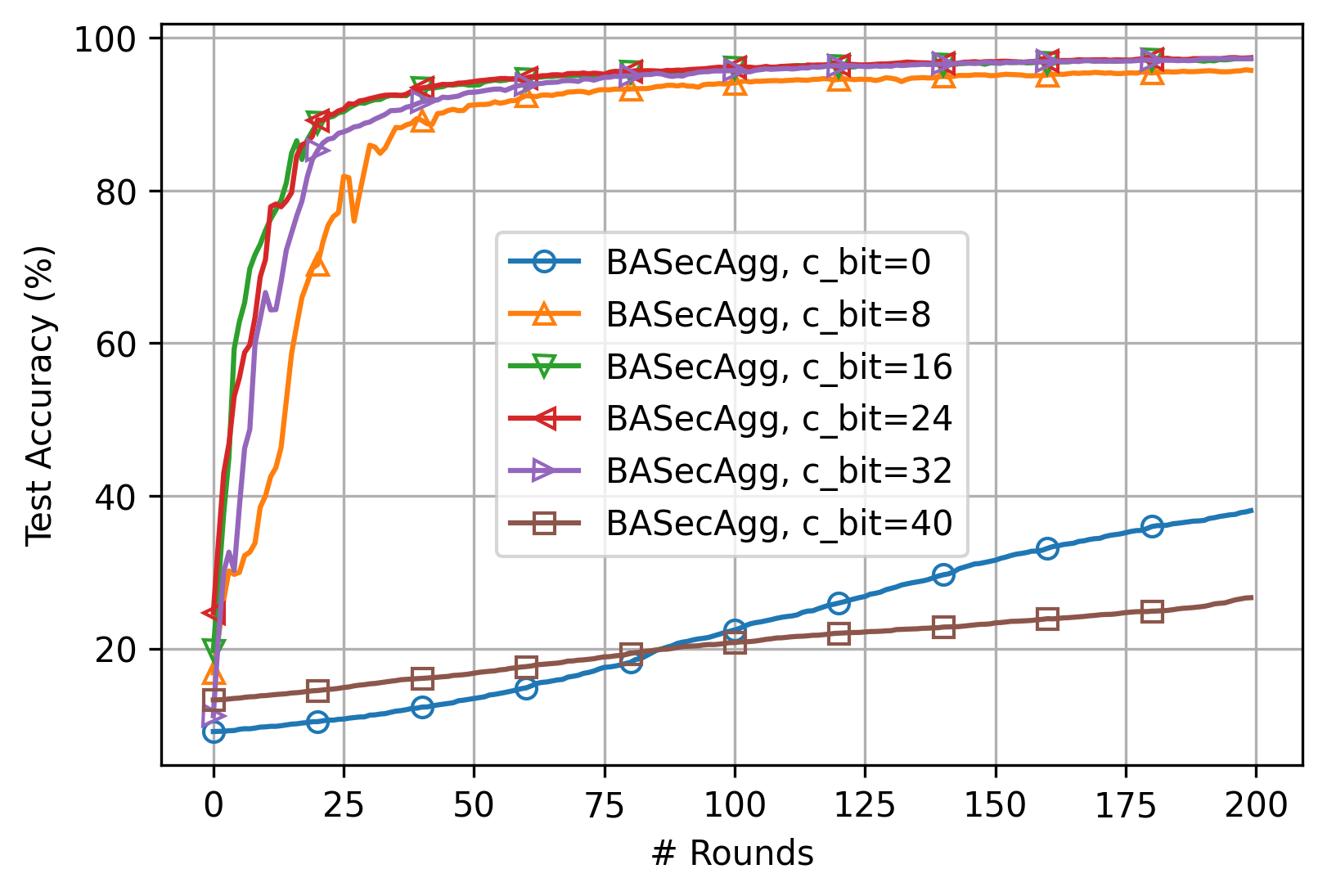}
    }
    \subfigure[CIFAR-$10$ dataset.]{\label{fig:CIFAR10_diff_q_level}
    \includegraphics[scale=0.45]{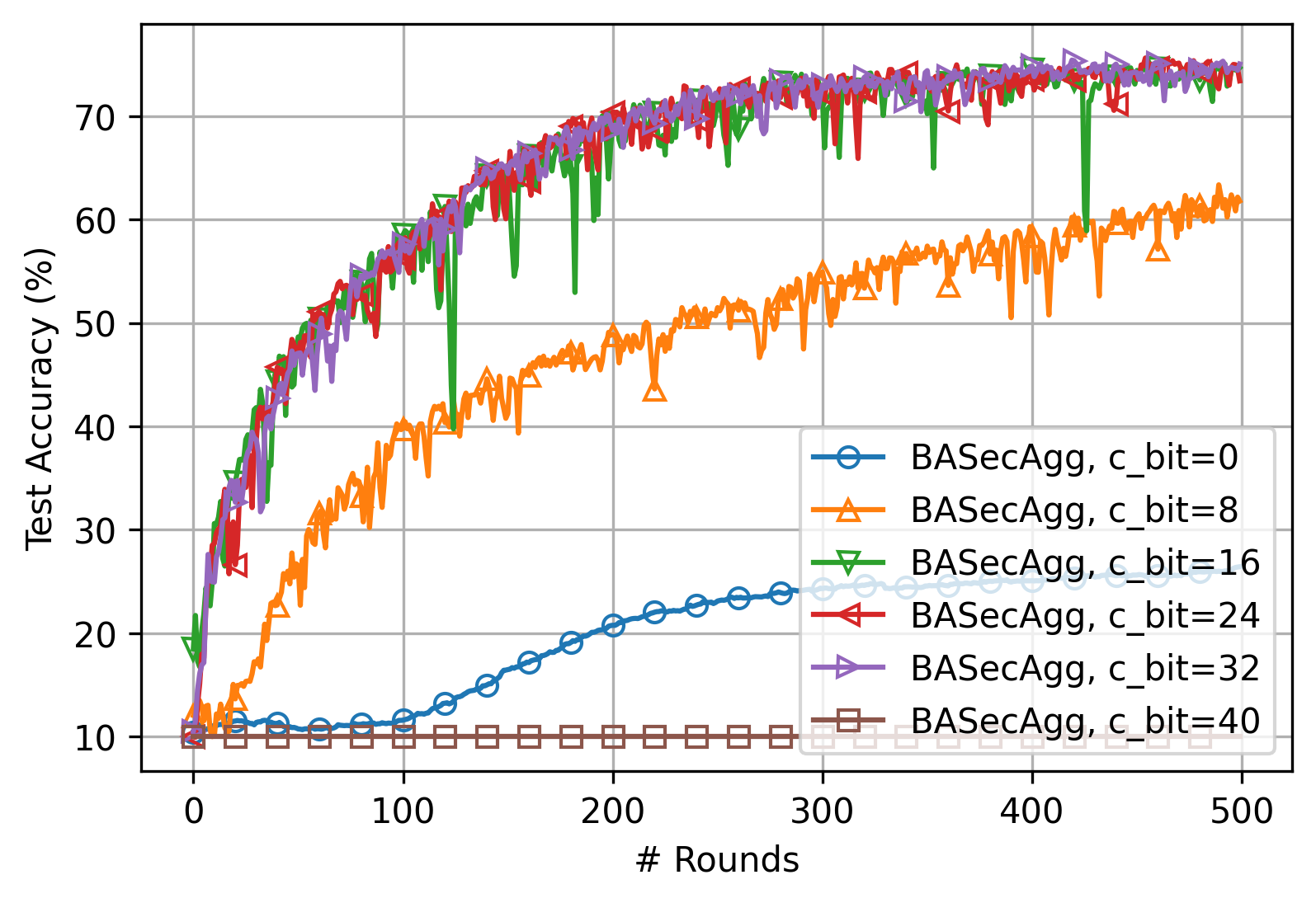}
    }
\vspace{-5 pt}
\caption{\footnotesize Accuracy of \namespace and \FedBuffspace with various values of quantization parameter $c_l = 2^{c_{bit}}$.}
\label{fig:diff_q_level}
\vspace{-0.15cm}\end{figure*}

\vspace{5 pt}
\noindent {\bf Performance with various quantization levels.}
To investigate the impact of the quantization, we measure the performance with various values of quantization parameter $c_l$ on MNIST and CIFAR-10 datasets in Fig. \ref{fig:diff_q_level}. 
We can observe that $c_l = 2^{16}$ has the best performance while small or large value of $c_l$ has the poor performance. 
This is because the value of $c_l$ provides a trade-off between two sources of quantization noise: 1) the rounding error from the stochastic rounding function defined in \eqref{eq:sto_round} and 2) the wrap-around error when modulo operations are carried out in the finite field. 
When $c_l$ has small value the rounding error is dominant while the wrap-around error is dominant when $c_l$ has large value. 
To find a proper value of $c_l$, we can utilize the auto-tuning algorithm proposed in \cite{bonawitz2019federated}. 


\section{Conclusions}
In this paper, we have proposed a buffered asynchronous secure aggregation protocol (\name) that is not based on TEEs. The independence of TEEs allows \namespace to have any buffer size unlike \FedBuff. The crux of \namespace is that it designs the masks of the users such that they cancel out in the buffer even if they belong to different training rounds. Our convergence analysis and experiments show that \namespace almost has the same convergence guarantees as \FedBuff. 

\bibliographystyle{plain}
\bibliography{main} 

\clearpage
\newpage

\appendix



\section{Theoretical Guarantees of \name: Proof of Theorem \ref{thm:convergence}}
\label{app:convergence_proof}
The proof of Theorem \ref{thm:convergence} directly follows from the following useful lemma that shows the unbiasedness and bounded variance still hold for the quantized gradient estimator $Q_c(g(\mathbfsl{x},\xi))$ for any  $\mathbfsl{x}\in\mathbb{R}^d$.
\begin{lemma}\label{lemma:q_prop}
    For the quantized gradient estimator $Q_c(g(\mathbfsl{x},\xi))$ with a given vector $\mathbfsl{x}\in\mathbb{R}^d$
    where $\xi$ is a uniform random variable representing the sample drawn, $g$ is a gradient estimator such that $\mathbb{E}_\xi [g(\mathbfsl{x},\xi)]=\nabla F(\mathbfsl{x})$ and $\mathbb{E}_\xi \lVert g(\mathbfsl{x},\xi) - \nabla F(\mathbfsl{x}) \rVert^2 \leq \sigma_l^2$, and the stochastic rounding function $Q_c$ is given in~\eqref{eq:sto_round}, the following holds,
    \begin{align}
        \mathbb{E}_{Q,\xi} [Q_c(g(\mathbfsl{x},\xi))] &= \nabla F(\mathbfsl{x}) \label{unbiased}\\
        \mathbb{E}_{Q,\xi} \lVert Q_c(g(\mathbfsl{x},\xi)) - \nabla F(\mathbfsl{x}) \rVert^2
        &\leq \sigma_c^{2}, \label{variance}
    \end{align}
    where $\sigma_c^2 = \frac{d}{4c^2} + \sigma^2_l$.
\end{lemma}

\begin{proof}
(Unbiasedness). 
Given $Q_c$ in~\eqref{eq:sto_round} and any random variable $x$, it follows that,
\begin{align}
    \mathbb{E}_Q \left[Q_c(x) \mid x\right] 
    =&\; \frac{\lfloor cx \rfloor}{c} \left(1-(cx-\lfloor cx \rfloor)\right) + \frac{(\lfloor cx \rfloor+1)}{c} (cx-\lfloor cx \rfloor) \notag \\
    =&\; x 
\end{align}
from which we obtain the unbiasedness condition in \eqref{unbiased},
\begin{align}
    \mathbb{E}_{Q,\xi} [Q_c(g(\mathbfsl{x},\xi))] 
    &= \mathbb{E}_{\xi}\big[ \mathbb{E}_{Q}[Q_c(g(\mathbfsl{x},\xi))\mid g(\mathbfsl{x},\xi)] \big] \notag \\
    &= \mathbb{E}_{\xi}\big[ g(\mathbfsl{x},\xi) \big] \notag \\
    &= \nabla F(\mathbfsl{x}). 
\end{align}

\noindent    
(Bounded variance). 
Next, we observe that,
\begin{align}
    &\mathbb{E}_Q \left[ \big(Q_c(x) - \mathbb{E}_Q[Q_c(x)\mid x]\big)^2 \mid x \right] \notag \\
    &\quad = \left(\frac{\lfloor cx \rfloor}{c} - x\right)^2 (1-(cx-\lfloor cx \rfloor)) 
    + \left(\frac{\lfloor cx \rfloor+1}{c}-x \right)^2(cx-\lfloor cx \rfloor) \notag \\
    &\quad = \frac{1}{c^2}\left(\frac{1}{4} - \left(cx-\lfloor cx \rfloor -\frac{1}{2}\right)^2 \right) \notag \\
    &\quad \leq \frac{1}{4c^2} \label{eq:lem1_ineq1}
\end{align}
from which we obtain the bounded variance condition in \eqref{variance} as follows,
\begin{align}
    &\mathbb{E}_{Q,\xi} \lVert Q_c(g(\mathbfsl{x},\xi)) - \nabla F(\mathbfsl{x}) \rVert^2 \notag \\
    &\quad = \mathbb{E}_{\xi}\big[ \mathbb{E}_{Q}[ \lVert Q_c(g(\mathbfsl{x},\xi)) - \nabla F(\mathbfsl{x}) \rVert^2 \mid g(\mathbfsl{x},\xi)] \big] \notag \\
    &\quad \leq \mathbb{E}_{\xi}\big[ \mathbb{E}_{Q}[ \lVert Q_c(g(\mathbfsl{x},\xi)) - g(\mathbfsl{x},\xi) \rVert^2 \mid g(\mathbfsl{x},\xi)] \big] 
     + \mathbb{E}_{\xi}\big[ \mathbb{E}_{Q}[ \lVert g(\mathbfsl{x},\xi) - \nabla F(\mathbfsl{x}) \rVert^2 \mid g(\mathbfsl{x},\xi)] \big] \label{eq:triangle_ineq} \\
    &\quad \leq \frac{d}{4c^2} + \sigma^2_l  \label{eq:lem1_ineq2} \\
    &\quad = \sigma^{2}_c, \notag
\end{align}
where \eqref{eq:triangle_ineq} follows from the triangle inequality and \eqref{eq:lem1_ineq2} follows form~\eqref{eq:lem1_ineq1}.
\end{proof}

Now, the update equation of \namespace is equivalent to the update equation of \FedBuffspace except that \namespace has an additional random source, stochastic quantization $Q_{c_l}$, which also satisfies the unbiasedness and bounded variance. 
One can show the convergence rate of \namespace presented in Theorem \ref{thm:convergence} by exchanging $\mathbf{E}_\xi$ and variance-bound $\sigma^2_l$ in \cite{nguyen2021federated} with $\mathbf{E}_{Q_{c_l}, \xi}$ and variance-bound $\sigma^2_{c_l}= \frac{d}{4{c_l}^2} + \sigma^2_l$, respectively.

\section{Experiment Details}
\label{app:exp_details}
In this appendix, we provide more details about the experiments of Section \ref{sec:Experiments}.

\noindent {\bf Hyperparameters.}
For all experiments, we tune the hyperparameters based on the validation accuracy for each dataset by partitioning $20\%$ of the training samples into the validation dataset. We use mini-batch SGD for all tasks with a mini-batch size of $50$.
We select the best parameters for the global learning rate $\eta_g$, local learning rate $\eta_l$, $L_2$ regularization parameter $\lambda$, and staleness exponent $\alpha$ with the following sweep ranges
\begin{align}
    \eta_g  & \in \{1.0, 0.1, 0.01\}, \notag \\
    \eta_l  & \in \{0.1, 0.03, 0.01, 0.003, 0.001\}, \notag \\
    \lambda & \in \{5e^{-3}, 5e^{-4}, 5e^{-5}\},  \notag \\
    \alpha  & \in \{0.1, 0.5, 1.0, 1.5, 2.0\}. \notag
\end{align}
We have found that the best values of $\eta_g$, $\lambda$, and $\alpha$ are $1.0$, $5e^{-4}$, and $1.0$, respectively for both MNIST and CIFAR-10 datasets. Finally, we have found that the best value of $\eta_l$ is $0.01$ and $0.1$ for MNIST and CIFAR-10 datasets, respectively.

\end{document}